\def\year{2022}\relax
\documentclass[letterpaper]{article} % DO NOT CHANGE THIS
\usepackage{aaai22}  % DO NOT CHANGE THIS
\usepackage{times}  % DO NOT CHANGE THIS
\usepackage{helvet}  % DO NOT CHANGE THIS
\usepackage{courier}  % DO NOT CHANGE THIS
\usepackage[hyphens]{url}  % DO NOT CHANGE THIS
\usepackage{graphicx} % DO NOT CHANGE THIS
\usepackage{natbib}  % DO NOT CHANGE THIS AND DO NOT ADD ANY OPTIONS TO IT
\usepackage{caption} % DO NOT CHANGE THIS AND DO NOT ADD ANY OPTIONS TO IT
\DeclareCaptionStyle{ruled}{labelfont=normalfont,labelsep=colon,strut=off} % DO NOT CHANGE THIS
\usepackage{algorithm}
\usepackage{algorithmic}

\usepackage{verbatim}
\usepackage{microtype}
\usepackage{subfigure}
\usepackage{booktabs} % for professional tables
\usepackage{url}            % simple URL typesetting
\usepackage{amsmath,amsthm,amssymb,amsfonts} 
\usepackage{microtype}      % microtypography
\usepackage{pifont}
\usepackage{nicefrac}      
\usepackage{mathtools}
\mathtoolsset{showonlyrefs}
\theoremstyle{plain}
\newtheorem*{theorem*}{Theorem}
\newtheorem{theorem}{Theorem}

\newtheorem{lemma}[theorem]{Lemma}

\newcommand{\bc}{\boldsymbol{c}}

\newcommand{\bg}{\boldsymbol{g}}
\newcommand{\bx}{\boldsymbol{x}}
\newcommand{\bu}{\boldsymbol{u}}
\newcommand{\by}{\boldsymbol{y}}

\newcommand{\bz}{\boldsymbol{z}}

\newcommand{\bv}{\boldsymbol{v}}

\newcommand{\bpsi}{\boldsymbol{\psi}}

\newcommand{\dom}{\mathop{\mathrm{dom}}}

\newcommand{\field}[1]{\mathbb{#1}}

\newcommand{\R}{\field{R}}

\newcommand{\E}{\field{E}}

\newcommand{\btheta}{\boldsymbol{\theta}}

\newcommand{\norm}[1]{\left\|{#1}\right\|}

\DeclareMathOperator{\Wealth}{Wealth}

\usepackage{newfloat}
\usepackage{listings}
\floatstyle{ruled}
\newfloat{listing}{tb}{lst}{}
\floatname{listing}{Listing}
\title{Better Parameter-Free Stochastic Optimization with ODE Updates for Coin-Betting}

\author {
    Keyi Chen\textsuperscript{\rm 1}, John Langford\textsuperscript{\rm 2}, Francesco Orabona\textsuperscript{\rm 1}
}
\affiliations {
    \textsuperscript{\rm 1} Boston University, Boston, MA  \\
    \textsuperscript{\rm 2} Microsoft Research, New York, NY\\
    keyichen@bu.edu, jcl@microsoft.com, francesco@orabona.com
}

\begin{document}

\maketitle

\begin{abstract}
Parameter-free stochastic gradient descent (PFSGD) algorithms do not require setting learning rates while achieving optimal theoretical performance. In practical applications, however, there remains an empirical gap between tuned stochastic gradient descent (SGD) and PFSGD. In this paper, we close the empirical gap with a new parameter-free algorithm based on continuous-time Coin-Betting on truncated models. The new update is derived through the solution of an Ordinary Differential Equation (ODE) and solved in a closed form. We show empirically that this new parameter-free algorithm outperforms algorithms with the ``best default'' learning rates and almost matches the performance of finely tuned baselines without anything to tune.
\end{abstract}

\section{Introduction}

Most machine learning algorithms require solving an optimization problem, $\min_{\bx \in
  \R^d} \ F(\bx)$. To solve this problem, first-order stochastic
optimization algorithms are the de-facto choice for machine learning
due to their speed across large datasets and simplicity. These Stochastic (sub)Gradient Descent (SGD) algorithms start from an initial solution $\bx_1$, iteratively update a vector $\bx_t$ moving in the negative direction of a stochastic (sub)gradient $\bg_t$ such that $\E[\bg_t]
\in \partial F(\bx_t)$: $\bx_{t+1}=\bx_t - \eta_t \bg_t$, where $\eta_t>0$ is the \emph{learning rate} or \emph{step size}. Learning rates are the big caveat of SGD.

How do we set the learning rate? Intuitively, the learning rates must become arbitrarily small to converge to the minimum of the function. This is clear considering minimizing the function $F(x)=|x-10|$ with SGD. In addition, the step size must be large enough that not too many updates are required to move from the initial to the optimal solution.

We can formalize the above intuitions with the standard convergence rate of SGD with constant step size $\eta$ after $T$ iterations with stochastic subgradients $\bg_t$ bounded by $1$ in $L_2$ norm~\citep{Zinkevich03}:
\begin{equation}
\label{eq:conv_sgd}
\E\left[F\left(\frac1T \sum_{t=1}^T \bx_t\right)\right] - F(\bx^\star)
\leq \frac{\|\bx_1-\bx^\star\|^2}{2\eta } + \frac{\eta T}{2}~.
\end{equation}
From the above, we have that the optimal worst-case step size is $\eta =\frac{\|\bx_1-\bx^\star\|_2}{\sqrt{T}}$ implying the optimal step size is inversely proportional to the square root of the number of iterations and proportional to the distance between the initial point and the optimal one $\bx^\star$. Unfortunately, we do not know in advance the distance from the initial point to the optimal solution nor we can expect to reliably estimate it---again, consider the function $F(x)=|x-10|$. This lack of information about $\|\bx_1-\bx^\star\|_2$ is the primary difficulty of choosing the learning rate in the stochastic setting.

From a practical point of view, this failure of the theory to provide a way to automatically set the learning rates means that most of the time the best way to achieve the best convergence rate is to treat the learning rate as a hyperparameter and exhaustively search for the best one. However, the computational cost of this search can be huge, basically multiplying the entire learning process by the number of different learning rates we have to try. 

However, a new class of \emph{parameter-free} algorithms has been recently proposed~\citep[e.g.][]{McMahanO14,Orabona14,Cutkosky16,OrabonaP16b,CutkoskyB17,OrabonaT17,FosterRS18,CutkoskyO18,Kotlowski19,KempkaKW19,cutkosky2019matrixfree,JunO19,MhammediK20,OrabonaP21}. These algorithms do not have a learning rate parameter at all, while achieving essentially the same theoretical convergence rate you would have obtained tuning the learning rate in \eqref{eq:conv_sgd}. The simplest parameter-free algorithm~\citep{OrabonaP16b} has an update rule of
\vspace{-4pt}
\begin{equation}
\label{eq:coin}
\bx_{t+1} = \frac{-\sum_{i=1}^t \bg_i}{L(1+t)} \left(1 - \sum_{i=1}^t \langle \bg_i, \bx_i \rangle\right)~,
\end{equation}
\vspace{-1pt}
where $L$ is the Lipschitz constant of $F$.
These algorithms basically promise to trade-off a bit of accuracy for the removal of tuning the learning rate. However, empirically they still have a big gap with tuned optimization algorithms.

\begin{figure}
\centering
\includegraphics[width=0.45\textwidth]{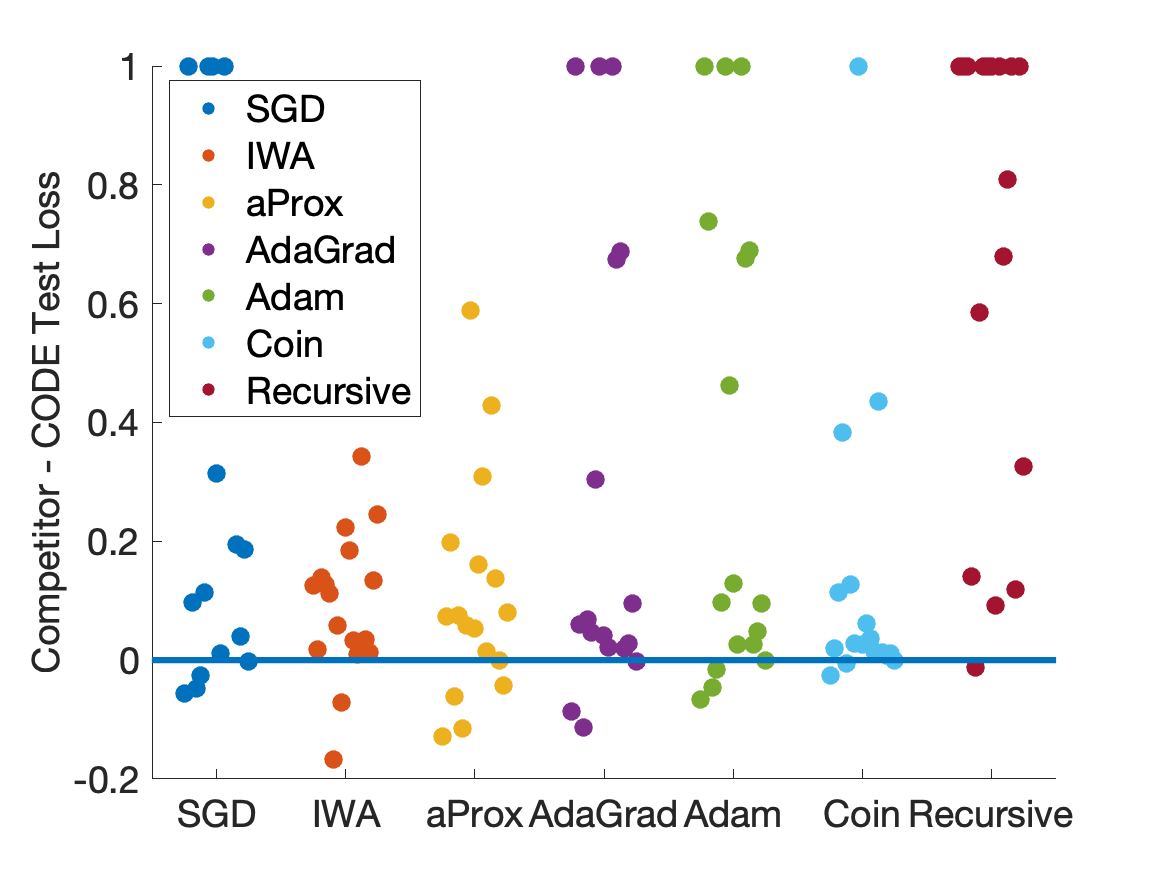}
\caption{Difference between the normalized competitor and CODE test losses on 17 regression datasets using ``best default'' parameters. Each point represents a dataset.  Points located above the line $y=0$ represent datasets on which CODE outperforms competitors.}
\label{fig:competitors}
\end{figure}

\textbf{Contributions.} In this paper, we greatly reduce this gap with a new class of parameter-free stochastic optimization algorithms that performs better than SGD, AdaGrad, and Adam with the ``best default'' parameter, see Figure~\ref{fig:competitors}. 
We achieve it by designing a parameter-free algorithm that is more aware of the geometry of the loss functions by avoiding overshooting their minima. In particular, we modify the optimization algorithm based on coin-betting in \eqref{eq:coin} to make an \emph{infinite number of infinitesimally small updates} on a \emph{truncated linear model for optimization}~\citep{AsiD19}. The final update is a closed form solution of an Ordinary Differential Equation (ODE), hence we call our algorithm CODE: Coin-betting ODE.

\textbf{Related Work.}
Parameter-free learning algorithms are discussed in Section~\ref{sec:def}.
The main inspiration here is the Importance Weight Aware updates~\citep{KarampatziakisL11} which make infinitely many infinitesimal SGD updates for each sampled loss. We provide two improvements: making the approach work for a parameter-free update rule and generalizing the set of losses. In particular, we consider any convex Lipschitz function instead of an expectation of losses of the form satisfying $\ell_t(\bx_t)=\ell(\langle \bz_t, \bx_t\rangle, y_t)$. Achieving closed-form solutions requires using a truncated linear model~\citep{AsiD19}.
The Importance Weight Aware updates~\citep{KarampatziakisL11} are also close to the ones in the truncated model from \citet{AsiD19}, perfectly coinciding in some cases. Both these approaches are also similar to Passive-Aggressive online learning algorithms~\citep{CrammerDKSSS06}, implicit updates~\citep{KivinenW97,KulisB10}, and the proximal point method~\citep{Rockafellar76}.

There is a long history relating ODE to optimization. The corresponding ODEs of a numerical optimization algorithm are established by taking infinitesimal step sizes so that the trajectory converges to a curve modeled by the ODE. The continuous trajectory provides new insights into numerical optimization such as the continuous-time interpretation of Nesterov's method~\citep{su2015differential} and the accelerated Mirror Descent inspired by continuous-time analysis~\citep{NIPS2015_5843}. This paper is the first attempt to combine a parameter-free algorithm with an ODE approach. The technique is general and we believe it points out an interesting new direction for research.

\textbf{Organization of the paper.} In Section~\ref{sec:def}, after introducing some definitions, we briefly review the theory behind parameter-free optimization algorithms through the simplified lens of coin-betting algorithms. Then, in Section~\ref{sec:ode} we introduce our algorithm CODE. Finally, in Section~\ref{sec:exp} we present empirical results and we conclude with a discussion and future work in Section~\ref{sec:conc}.

\section{Preliminaries}
\label{sec:def}

In this section, we introduce the needed mathematical background and the basic idea of parameter-free coin-betting optimization algorithms and truncated linear models.

\textbf{Notation.}
We denote vectors by bold letters and matrices by capital letters, e.g. $\bx \in \R^d$ and $A \in \R^{d\times m}$.
We denote by subscript $t$ a variable that changes in a discrete way, while using a function of $t$ for a variable that changes over time in a continuous way, e.g., $\bx_t$ and $\bx(t)$. $\boldsymbol{1}[E]$ denotes the indicator function of an event $E$, i.e., $\boldsymbol{1}[E]$ is equal to 1 if $E$ is true and 0 otherwise.

\textbf{Convex Analysis.}
We denote by $\norm{\cdot}$ the $L_2$ norm in $\R^d$. Let $f: \R^d \to \R
\cup \{\pm\infty\}$, the \emph{Fenchel conjugate} of $f$ is $f^\star:\R^d \to \R \cup \{\pm
\infty\}$ with $f^\star(\btheta) = \sup_{\bx \in
\R^d} \ \btheta^\top \bx - f(\bx)$.

A vector $\bx$ is a \emph{subgradient} of a convex function $f$ at $\bv$ if $f(\bv) - f (\bu) \leq  (\bv - \bu)^\top \bx$ for any $\bu \in \dom f$. The \emph{differential set} of $f$ at $\bv$, denoted by $\partial f(\bv)$, is the set of all the subgradients of $f$ at $\bv$. If $f$ is also differentiable at $\bv$, then $\partial f(\bv)$ contains a single vector, $\nabla f(\bv)$, which is the \emph{gradient} of $f$ at $\bv$.

\textbf{Betting on a coin.}
We describe here how to reduce subgradient descent to betting on the outcome of a binary event (i.e. a coin flip). This is not an exhaustive review of the topic---interested readers are referred to, e.g., \citet{OrabonaP16b} and \citet[Chapter 9]{Orabona19}.

We consider a gambler making repeated bets on the outcomes of adversarial coin flips. The gambler starts with \$1. In each round $t$, he bets on the outcome of a coin flip $c_t \in \{-1,1\}$, where $+1$ denotes heads and $-1$ denotes tails. We do not make any assumption on how $c_t$ is generated.

The gambler can bet any amount, although no additional money may be borrowed. We encode the gambler's bet in round $t$ by $x_t \in \R$, where $sign(x_t)$ encodes whether the bet is on heads or tails and $|x_t|$ encodes the betted amount.  When the bet succeeds, the gambler wins $x_tc_t$, otherwise, $x_tc_t$ is lost.
 We define
$\Wealth_t$ as the gambler's wealth at the end of round $t$, that is
\vspace{-10pt}
\begin{equation}
\label{equation:def_wealth_reward}
\Wealth_t = 1 + \sum_{i=1}^t x_i c_i~.
\end{equation}
\vspace{-6pt}

We enforce $x_t = \beta_t \Wealth_{t-1}$ for some betting fraction $\beta_t \in [-1,1]$ implying that the gambler cannot borrow money. We also slightly generalize the problem by allowing the outcome of the coin flip $\bc_t$ to be a vector in $\R^d$ with $L_2$ norm bounded by 1, with the definition of the wealth in~\eqref{equation:def_wealth_reward} generalized through inner products.

Now, we give a proof sketch of how it is possible to reduce optimization to a coin-betting algorithm.
Consider the function $F(x):=|x-10|$ and the optimization problem $\min_{x} \ F(x)$.

We set the outcome of the coin flip $c_t$ to be equal to the negative subgradient $g_t$ of $F$ in $x_t$, that is $c_t=-g_t \in \partial [-F(x_t)] \in \{-1, 1\}$, where $x_t$ is the bet.

Let's also assume that there exists a function $H(\cdot)$ such that our betting strategy guarantees that $\Wealth_T$ is at least $H(\sum_{t=1}^T c_t)$ for any arbitrary sequence $g_1, \cdots, g_T$.

We claim that the average of the bets, $F\left(\tfrac{1}{T}\sum_{t=1}^T x_t\right)$, converges to the minimum, $F(x^\star)$, with a rate dependent on how good our betting strategy is. In fact, we have
\vspace{-6pt}
\begin{align*}
F\left(\frac{1}{T} \sum_{t=1}^T x_t\right)& - F(x^\star) 
\leq \frac{1}{T} \sum_{t=1}^T F(x_t) - F(x^\star) \\
&\leq \frac{1}{T}\sum_{t=1}^T -g_t x^\star -\frac{1}{T}\sum_{t=1}^T -g_t x_t \\
&\leq \tfrac{1}{T}+\tfrac{1}{T}\left(\sum_{t=1}^T c_t x^\star  - H\left(\sum_{t=1}^T c_t\right)\right) \\
&\leq \tfrac{1}{T}+\tfrac{1}{T}\max_v \ v x^\star  - H(v) 
= \tfrac{H^\star(x^\star)+1}{T},
\end{align*}
\vspace{-2pt}
where in the first inequality we used Jensen's inequality, in the second the definition of subgradients, 
in the third our assumption on $H$, and in the last equality the definition of Fenchel conjugate of $H$.

In words, we can use any gambling algorithm to find the minimizer of a non-smooth objective function by accessing its subgradients. Note that the outlined approach works in any dimension for any convex objective function, even with stochastic subgradients~\citep{OrabonaP16b}. In particular, using the Krichevsky-Trofimov (KT) betting strategy that sets the signed betting fraction at time $t$ equal to $\beta_t = \tfrac{\sum_{i=1}^{t-1} c_t}{t}$~\citep{KrichevskyT81}, \citet{OrabonaP16b} obtained the optimal parameter-free algorithm in \eqref{eq:coin}.

\textbf{Truncated Linear Models.}
It is well-known that first-order optimization algorithms for convex functions can be thought as minimizing surrogate lower bounds to the objective function $F$, instead of directly minimizing it. That is, at step $t$ instead of trying to minimize $F(\bx)$, we minimize the first order approximation of $F$ in $\bx_t$, i.e., $F(\bx_t)+ \langle \bg_t, \bx-\bx_t\rangle$ where $\bg_t \in \partial F(\bx_t)$. Of course, this is a linear function with a minimum at infinity. Hence, we minimize the function \emph{constrained to a local neighborhood of $\bx_t$ w.r.t. the $L_2$ norm}. Using a Lagrangian multiplier, we obtain exactly the subgradient descent algorithm.

This view of subgradient descent as minimizing local lower bounds to the function immediately gives us a way to obtain new algorithms. For example, changing the metric of the constraint, we go from subgradient descent to Mirror Descent~\citep{BeckT03}. More recently, \citet{AsiD19} proposed to substitute the linear approximation with a \emph{truncated linear approximation}, that is
\begin{equation}
\label{eq:truncated}
\tilde{F}_t(\bx)=\max[F(\bx_t) + \langle \bg_t, \bx- \bx_t\rangle,F_{-}],
\end{equation}
where $F_{-}$ is a lower bound to the value of the function. For example, we can set $F_{-}$ to 0 if we know that the objective function is non-negative. Coupling this local lower bound with the constraint of staying not too far from $\bx_t$ in a $L_2$ sense, \citet{AsiD19} obtained a new class of optimization algorithms, aProx, that are more robust to the setting of the learning rate in low-noise regimes.

\section{ODE Updates for Coin-Betting Optimization}
\label{sec:ode}

In this section, we describe our Coin-betting ODE (CODE) algorithm.
The main idea is to discretize each update of the parameter-free coin betting algorithm \eqref{eq:coin} in an infinite series of infinitely small updates. Moreover, thanks to the use of the truncated linear model described in the previous section, we are still able to recover a closed-form solution for each update. 
It is worth noting that a direct application of the truncation method would require a mirror descent optimization algorithm, but betting does not belong to the mirror descent family. Indeed, there are no direct ways to merge truncation and betting techniques, other than the ODE's approach we propose in the following.

We proceed in stages towards the closed-form update. First, note that parameter-free algorithms require the stochastic gradients $\bg_t$ to be bounded, see lower bound in~\citep{CutkoskyB17}. For simplicity, we set this bound to be equal to 1.
Now, we introduce a straightforward improvement of the update rule \eqref{eq:coin}:
\vspace{-8pt}
\begin{equation}
\label{eq:coin2}
\bx_{t+1} = \tfrac{-\sum_{i=1}^t \bg_i}{1+\sum_{i=1}^t \boldsymbol{1}[\bg_i\neq\boldsymbol{0}]} \left(1 - \sum_{i=1}^t \langle \bg_i, \bx_i \rangle\right)~.
\end{equation}
\vspace{-2pt}
This update gives an improved convergence rate that depends on $O(T^{-1}\sqrt{\sum_{t=1}^T \boldsymbol{1}[\bg_t\neq\boldsymbol{0}]})$, where $\bg_t\in \partial F(\bx_t)$, rather than $O(T^{-1/2})$. We defer the proof of this result to the Appendix.

\begin{figure}
\centering
\includegraphics[width=0.44\textwidth]{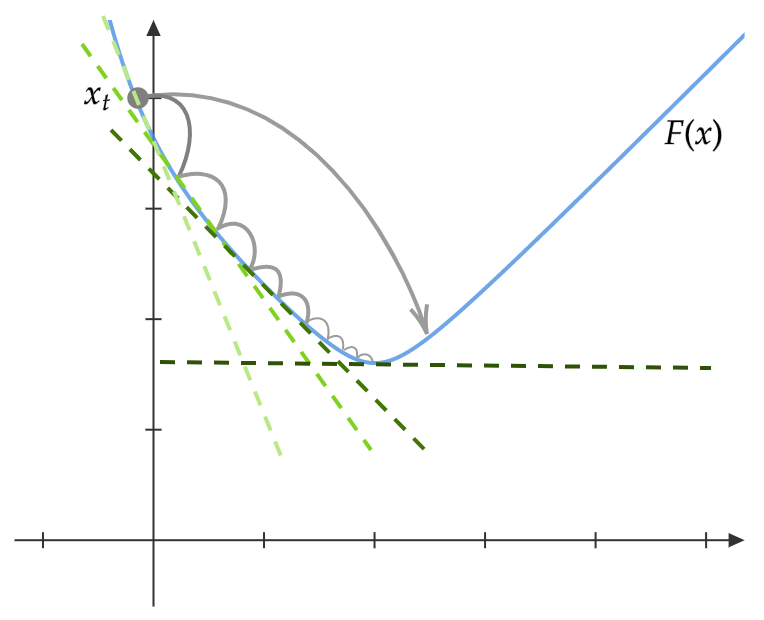}
\caption{ODE updates versus standard updates.}
\label{fig:trunc}
\end{figure}

Second, let's explain the idea behind having infinitely many steps. Consider an online convex optimization algorithm that at time $t$ predicts $\bx_t$, receives the subgradient, and then updates its prediction as $\bx_{t+1}$.  What if we instead make a step of 1/10th of the original size, we receive a subgradient, update again with 1/10 of the step, and so on for 10 times? In subgradient descent, the update in the first case can end up far from the minimum of the original objective function due to overshooting. With 10 updates of 1/10th size, we instead expect the optimization algorithm to land much closer to the minimum of the objective function using the informational advantage of 10 subgradients with 10 small steps. See Figure~\ref{fig:trunc} for a graphic representation of this idea. The exact same phenomenon occurs in coin-betting optimization algorithms. In CODE, by using truncated linear models, we get additional gradient information at the corner of the hinge. 

However, in the above example, it is clear that considering 10 updates with weights 1/10 is arbitrary. Indeed, we could push this reasoning to the limit and have $1/\delta$ updates over the losses $\delta \ell_t$ when $\delta$ goes to zero. In this case, the algorithm follows a \emph{continuous trajectory} rather than a discrete one.
While this reasoning is compelling, calculating the trajectory along the gradient flow is computationally difficult. For example, with 10 updates of 1/10th size, the algorithm will request 10 gradients.

A first idea to solve this problem was proposed by \citet{KarampatziakisL11}, considering only objective functions of the form $F(\bx)=\E_t[\ell_t(\bx)]$, where $\ell_t(\bx)=\ell(\langle \bz_t, \bx\rangle,y_t)$. In this special case, only the derivative of $\ell$ changes under updates, not the direction of the gradient. Taking advantage of this, it is possible to calculate the final $\bx_{t+1}$ without computing the entire trajectory with additional information at the truncated linear models' corners. In other words, we do not need additional gradients to derive $\bx_{t+1}$. 

So, the last key component is to use the truncated linear models described in Section~\ref{sec:def}. In particular, if we use the updates in \eqref{eq:coin2} on the function $\tilde{F}_t$ in \eqref{eq:truncated}, the subgradient of $\tilde{F}_t$ always has the same direction and it is zero when we are in the flat part of the function. Hence, using \eqref{eq:coin2} with infinitesimally small weights of the function $\tilde{F}_t$ the update could never go beyond the point where $\tilde{F}_t(\bx)=F_{-}$.
Moreover, a closed-form solution along the continuous trajectory becomes possible.

Now that we have all the pieces, we state the closed-form solution that gives the final update rule.

\begin{theorem}
\label{thm:main}
Set $\Wealth_1=1$ and let $F(\bx)$ be a 1-Lipschitz convex function. 
Define $\btheta_1=\boldsymbol{0} \in \R^d$, $H_1=1$, and $\bg_t \in \partial F(\bx_t)$ for $t=1, \dots, T$.
Then, the limit as $\delta$ approaches 0 of running \eqref{eq:coin2} over $1/\delta$ updates with the surrogate functions $\delta \tilde{F}_t$ gives the update rule $\bx_{t+1} = \bpsi(t,h_t)$, where
\vspace{-14pt}

\begin{align}
\label{eq:flow_x}
&\bpsi(t,h) := \\
&\tfrac{\Wealth_t e^{-\langle \bg_t,\btheta_{t} \rangle \ln\left(1+\frac{h}{H_{t}}\right)+\|\bg_{t}\|^2 \left(h+H_{t}\ln\frac{H_{t}}{H_{t}+h}\right)}}{H_t+h}\left(\btheta_t - h \bg_t\right),\\
&h_t:=\min(\tilde{h}_t,1) \nonumber,
\end{align}
\vspace{-5pt}
the following quantities are defined recursively as 
\begin{align*}
&\btheta_{t+1}:=\btheta_{t} - h_t \bg_t, \qquad H_{t+1} := H_{t} + h_{t},\\
&\Wealth_{t+1} := \\
&\qquad \Wealth_t e^{-\langle \bg_t,\btheta_{t} \rangle \ln\left(1+\frac{h_t}{H_{t}}\right)+\|\bg_{t}\|^2 \left(h_t+H_{t}\ln\frac{H_{t}}{H_{t}+h_t}\right)},
\end{align*}

\vspace{-6pt}
and $\tilde{h}_t$ is the zero of the function 
\vspace{-2pt}

\begin{equation}
\label{eq:phi}
\phi(h):=F(\bx_t)+\langle \bg_t, \bpsi(t,h)- \bx_t\rangle - F_{-}~.
\end{equation}
\end{theorem}

Hence, to find the closed formula of the update, we only need to find the value of $h_t$ in each round. This depends on the zero of a one-dimensional function, hence it can be found with any standard algorithm, like bisection or the Newton algorithm.
Given the closed-form update, we can now state our CODE algorithm, Algorithm~\ref{alg:code}.
Note that, given the value of $h_t$, the computational complexity of each update is $O(d)$, like in SGD.

\begin{algorithm}[t]
\caption{Coin-betting ODE (CODE) Algorithm}
\label{alg:code}
\begin{algorithmic}[1]
\STATE Initialize: $\Wealth_0 = 1$, $H_1 = 1$, $\btheta_1=\boldsymbol{0} \in \R^d$
\FOR{$t = 1, \dots,T$}
\STATE Query point $\bx_{t} = \frac{\Wealth_{t}}{H_{t}}\btheta_{t}$
\STATE Receive $\bg_{t}$ such that $\E[\bg_t] \in \partial F(\bx_t)$, $\|\bg_{t}\|\leq 1$ 
\STATE Calculate $h_t=\min(1,\tilde{h}_t)$, where $\tilde{h}_t$ is the zero of the function $\phi$ in \eqref{eq:phi}
\STATE Update $\Wealth_{t+1} = \Wealth_{t} e^{-\langle \bg_t,\btheta_{t} \rangle \ln\left(1+\frac{h_t}{H_{t}}\right)+\|\bg_{t}\|^2 \left(h_t+H_{t}\ln\frac{H_{t}}{H_{t}+h_t}\right)}$
\STATE Update $H_{t+1}=H_{t}+h_t$
\STATE Update $\btheta_{t+1}=\btheta_{t} - h_t \bg_{t}$
\ENDFOR
\end{algorithmic}
\end{algorithm}

\paragraph{Proof of the update rule.}
To obtain the closed form solution in Theorem~\ref{thm:main}, we break each update of the coin-betting optimization algorithm into $1/\delta$ ``mini-updates'' over the truncated surrogate functions $\delta\tilde{F}_i$, then we take the limit for $\delta$ that goes to 0 and derive the final update rule in Theorem~\ref{thm:main}. This means that we have $1/\delta$ mini-updates between $\bx_t$ and $\bx_{t+1}$, that give us the points $\bx_{t+\delta}, \bx_{t+2 \delta}, \dots, \bx_{t+1-\delta}, \bx_{t+1}$. The following lemma fully characterizes these mini-updates. 

\begin{lemma}
\label{lemma:discr}
Assume that at time $t$ the Wealth of the coin-betting optimization algorithm is $\Wealth_t$, the number of non-zero subgradients received is $H_t$, and the sum of the previous weighted subgradients is $\btheta_t$. Receive the subgradient $\bg_t$, where $\|\bg_t\|\leq 1$. Then, breaking the update in \eqref{eq:coin2} in $1/\delta$ mini-updates with weights $\delta$ over the truncated linear model $\tilde{F}_t(\bx_t)$ is equivalent to the updates
%\vspace{-4pt}
\begin{align*}
&\Wealth_{t+1} = \Wealth_{t} + w_{t+1},\\ 
&\btheta_{t+1} = \btheta_t - \delta \bg_t \sum_{j=\delta}^{1-\delta} s_j, 
\qquad H_{t+1} &= H_t + \delta\sum_{j=\delta}^{1-\delta} s_j,
\end{align*}
%\vspace{-4pt}
where for any $j = 0, \delta, 2\delta, \dots, 1-\delta$, we define $w_t:=0$, $s_j := \boldsymbol{1}[\tilde{F}_t(\bx_{t+j}) \neq F_{-}]$, and $w_{t+j+\delta} := w_{t+j} - \delta s_j \langle \bg_t, \bx_{t+j}\rangle$.
\end{lemma}

\begin{proof}
We use the fact that $\tilde{F}_t$ is differentiable everywhere but in $\bx_t$ where we choose as subgradient the zero vector. So, overall the subgradient can only assume the value $\bg_t$ and $\boldsymbol{0}$.
In particular, the subgradient is the null vector iff $s_j=0$. This proves the updates of $w_t$ and $\btheta_t$, while $\btheta_t$ and $H_t$ are updated accordingly to the update rules in \eqref{eq:coin2}.
\end{proof}
\vspace{-2pt}

We now consider the case when $\delta$ goes to 0 and prove the Theorem~\ref{thm:main}.
 \vspace{-7pt}
\begin{proof}[Proof of Theorem~\ref{thm:main}]
We now consider the limit of the previous mini-updates in Lemma~\ref{lemma:discr} when $\delta$ goes to zero.
We immediately obtain that 
\vspace{-6pt}
\begin{align}
w_t'(j) 
&:= \lim\limits_{\delta\to 0} \frac{w_{t+j+\delta}-w_{t+j}}{\delta} \nonumber\\
&= \lim\limits_{\delta\to 0} -\frac{\delta \boldsymbol{1}[\tilde{F}_t(\bpsi(t,j)) \neq F_{-}] \langle \bg_t, \bpsi(t,j)\rangle}{\delta}\nonumber \\
&= -\boldsymbol{1}[\tilde{F}_t(\bpsi(t,j)) \neq F_{-}] \langle \bg_t, \bpsi(t,j)\rangle, \label{eq:deriv_w}
\end{align}
where $\bpsi$ is some function that represents the continuous trajectory of the iterates. In particular, $\bx_t=\bpsi(t,0)$.
Later, we will prove that $\bpsi$ coincides with the one we defined in \eqref{eq:flow_x}.
Also, defining $h_t$ as 
\vspace{-4pt}
\[h_t:=\lim\limits_{\delta\to 0} \ \delta \sum_{j=\delta}^{1-\delta} s_j 
= \int_{0}^1 \! \boldsymbol{1}[\tilde{F}_t(\bpsi(t,j)) \neq F_{-}] \, \mathrm{d}j,\]
we have
\vspace{-6pt}
\begin{align*}
&\lim\limits_{\delta\to 0} \ \btheta_t + \delta\sum_{j=\delta}^{1-\delta} s_j \bg_t = \btheta_t + h_t \bg_t,\\
&\lim\limits_{\delta\to 0} \ H_t+ \delta \sum_{j=\delta}^{1-\delta} s_j =H_t + h_t~.
\end{align*}
\vspace{-2pt}
Hence, using the above results in \eqref{eq:coin2}, we obtain the trajectory of $\bx_t$ to $\bx_{t+1}$ is described by
\[\bpsi(t,j)=\frac{\Wealth_t+w_t(j)}{H_t+j}(\btheta_j - j \bg_t).\]
Together with \eqref{eq:deriv_w}, this implies that $w'_t(j)=0$ for $j\geq h_t$, while for $j\leq h_t$ we have
\begin{equation}
\label{eq:w_prime}
w'_t(j)
=-\left\langle \bg_t,\frac{\Wealth_t+w_t(j)}{H_t+j}(\btheta_j - j \bg_t)\right\rangle~.
\end{equation}
To simplify the notation, denote by 
\[
P(j)=\frac{1}{H_t+j}\langle \bg_t,\btheta_t - j \bg_t\rangle
\]
and 
\[
Q(j)=-\frac{\Wealth_t}{H_t+j}\langle \bg_t,\btheta_t-j\bg_t\rangle~.
\] 
Note that $Q(j)= - \Wealth_j P(j)$.
Hence, we can rewrite \eqref{eq:w_prime} as $w'_t(j) + w_t(j)P(j) = Q(j)$.
Solving this first-order, linear, inhomogeneous ODE, we get
\begin{align*}
    w'_t(j)& e^{\int \! P(j)\,\mathrm{d}j} =\left[\int \! e^{\int \! P(j)\,\mathrm{d}j}Q(j)\,\mathrm{d}j+C\right]\\
    &=-\Wealth_t \int \! e^{\int \! P(j)\,\mathrm{d}j}d\left({\int \! P(j)\,\mathrm{d}j}\right)+C\\
    &=-\Wealth_t e^{\int \! P(j)\,\mathrm{d}j}+C,
\end{align*}
where $C$ is a constant. Next, we need to solve for $C$. Consider that
\begin{align*}
  &\int_{0}^{h_t} \! P(j)\,\mathrm{d}j
  =\int_{0}^{h_t} \! \frac{1}{H_t+j}\langle \bg_t,\btheta_t-j \bg_t\rangle \,\mathrm{d}j \\
  &=\langle \bg_t,\btheta_t\rangle \int_{0}^{h_t} \! \frac{1}{H_t+j} \,\mathrm{d}j - \|\bg_t\|^2\int_{0}^{h_t} \! \frac{j}{H_t+j} \,\mathrm{d}j\\
  &=\langle \bg_t,\btheta_t\rangle \ln(H_t+h_t) - \|\bg_t\|^2 (h_t-H_t\ln(H_t+h_t))~.
\end{align*}
%When $h=0$,$-\int P(j)dh=\langle y_tx_t,\theta_t\rangle \ln(H_t)+\|x_t\|^2(-H_t\ln(H_t))$
Hence, we have 
\begin{equation}
\begin{aligned}
w_t(h_t)&=-\Wealth_t\\
 &+ C e^{-\langle \bg_t,\btheta_t\rangle \ln(H_t+h_t) + \|\bg_t\|^2 (h_t-H_t\ln(H_t+h_t))}.\nonumber
\end{aligned}
\end{equation}
Since $w_t(0)=0$, we have 
\[C=\Wealth_t e^{\langle \bg_t,\btheta_t\rangle \ln(H_t) + \|\bg_t\|^2 H_t\ln(H_t))}.\]
Finally, we have
\begin{align*}
    \frac{w_t(h_t)}{\Wealth_t}= e^{-\langle \bg_t,\btheta_{t} \rangle \ln\left(1+\frac{h_t}{H_{t}}\right)+\|\bg_{t}\|^2 \left(h_t+H_{t}\ln\frac{H_{t}}{H_{t}+h_t}\right)} - 1~.
\end{align*}
Using the fact that $\frac{\Wealth_{t+1}}{\Wealth_t}=1 + \frac{w_t(h_t)}{\Wealth_t}$, we have the closed form expression of the wealth. This also provides an expression of the evolution of $\bx_t$ to $\bx_{t+1}$, that coincides with $\bpsi(t,j)$ in \eqref{eq:flow_x}.
\end{proof}

\section{Empirical Evaluation}
\label{sec:exp}
Here, we compare CODE with SGD, SGD with truncated models (aProx)~\citep{AsiD19}, SGD with Importance Weight Aware updates (IWA)~\citep{KarampatziakisL11}, AdaGrad~\citep{DuchiHS11}, Adam~\citep{KingmaB15}, the coin-betting algorithm in \eqref{eq:coin} (Coin)~\citep{OrabonaP16b} and the recursive coin-betting algorithm (Recursive)~\citep{cutkosky2019matrixfree}.
For SGD, aProx and IWA, we use the optimal worst-case step size for stochastic convex optimization: $\eta_k = \eta_0 /\sqrt{k}$, and tune the initial step size $\eta_0$.
In the adaptive learning rate methods, AdaGrad and Adam, we tune the initial step size $\eta_0$. CODE, Coin and Recursive do not have learning rates.

\subsection{Train/Test on Real Datasets}
We test the ability of CODE to get a good generalization error.
Hence, we perform experiments with 21 different machine learning binary classification datasets and 17 regression datasets from the LIBSVM website~\citep{ChangL01} and OpenML\citep{OpenML2013}. We implement extensive experiments on a large number of datasets to verify the significance of our results. We pre-process the samples normalizing them to unit norm vectors.
We shuffle the data and use 70\% for training, 15\% for validation, and hold out 15\% for testing. Given the lack of a regularizer, all the algorithms pass on the training set once to avoid overfitting~\citep[see, e.g., Section 14.5.1][]{Shalev-ShwartzBD14}. We evaluate algorithms with 0-1 loss for classification tasks and absolute loss for regression tasks, normalizing the scores by the performance of the best constant predictor. In this way, each dataset is weighted equally independently by how hard it is. Otherwise, a single hard dataset would dominate the average loss.
%Otherwise, easy datasets with smaller test loss would weight less than the difficult ones.
All the experiments are repeated 3 times and we take the mean of the 3 repetitions. See Appendix for more details on datasets, experiments, and numerical values.

\begin{figure}
\centering
\includegraphics[width=0.44\textwidth]{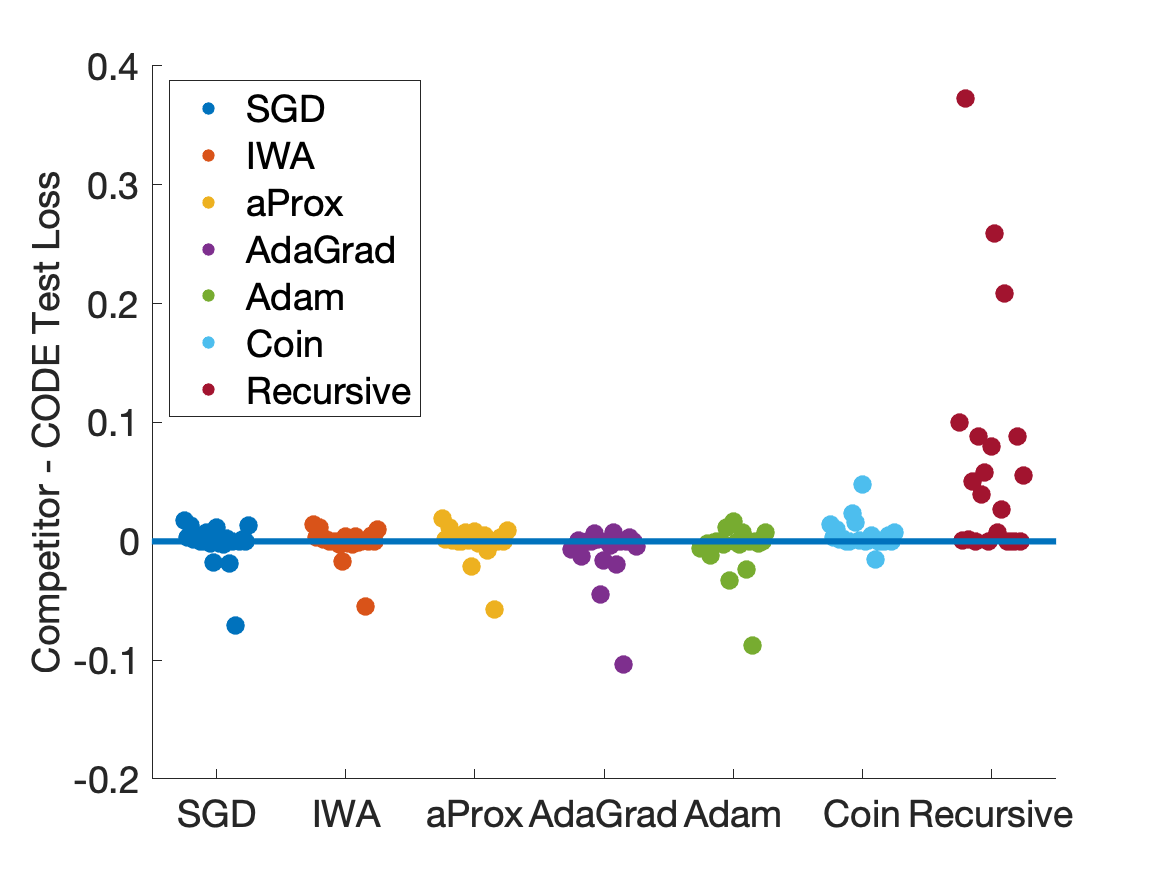}
\caption{Difference between normalized competitor and CODE test losses using the ``best default'' parameter on 21 classification datasets. }
\label{fig:competitors2}
\end{figure}

\vspace{-2pt}
\paragraph{Best Fixed Learning Rates.}

\begin{table}[t]
\centering
\begin{tabular}{lcc}
\toprule
 & \multicolumn{2}{c}{Single Learning Rate}  \\
\cmidrule{2-3}
Algorithm & \begin{tabular}{@{}c@{}}Mean normalized \\ absolute loss\end{tabular} & \begin{tabular}{@{}c@{}}Mean normalized \\ 0-1 loss\end{tabular} \\
\midrule
SGD     & 1.2116          & 0.1853         \\
IWA     & 0.8744          & 0.1861          \\
aProx   & 0.8901          & 0.1862          \\
AdaGrad & 1.1227          & \textbf{0.1778} \\
Adam    & 1.3319          & 0.1808          \\
Coin    & 0.9809          & 0.1930          \\
Recursive &10.1144		 &0.2556		 	\\
CODE    & \textbf{0.7825} & 0.1872          \\
\bottomrule
\end{tabular}
\caption{Average normalized test set accuracies on 17 regression and 21 classification datasets with best fixed learning rates.}
\label{table:table1}
\end{table}

\begin{table}[t]
\centering
\begin{tabular}{lcc}
\toprule
 & \multicolumn{2}{c}{One Learning Rate per Dataset} \\
\cmidrule{2-3}
Algorithm & \begin{tabular}{@{}c@{}}Mean normalized \\ absolute loss\end{tabular} & \begin{tabular}{@{}c@{}}Mean normalized \\ 0-1 loss\end{tabular}\\
\midrule
SGD     & 0.7276 & 0.1822\\
IWA     & 0.7196 & 0.1800\\
aProx   & 0.7284 & 0.1815\\
AdaGrad & \textbf{0.7085} & \textbf{0.1694}\\
Adam    & 0.7089 & 0.1704\\
Coin    & 0.9809 & 0.1930\\
Recursive &10.1144 &0.2556\\
CODE    &0.7825  & 0.1872 \\
\bottomrule
\end{tabular}
\caption{Average normalized test set accuracies on 17 regression and 21 classification datasets with tuned learning rates.}
\label{table:table2}
\end{table}

Clearly, the process of tuning hyperparameters is computationally expensive. For example, if we want to try 2 different learning rates, we have to run SGD twice per dataset. Hence, to have a fair comparison in terms of computational cost, here we consider this setting: we test the common belief that many optimization algorithms have a ``default'' learning rate that works on every dataset. If this were true, tuning would not be an extra cost.
To test this scenario, we tune the learning rate of the baselines to achieve the best \emph{average of normalized performance over all datasets directly on the test sets}. That is, we choose the ``best default'' parameter of each algorithm to minimize the numbers in Table~\ref{table:table1}. This is strictly better for the baselines than choosing some fixed default parameters for each of them.

First, we compare all the algorithms on linear regression problems with the absolute loss. We summarize the results in Figure~\ref{fig:competitors} in the Introduction and in Table~\ref{table:table1}. In the figure, each point represents a baseline algorithm (x-axis) vs. the normalized test loss difference between the algorithm and CODE (y-axis) on one dataset. So, points located above $y=0$ represent datasets where CODE outperforms the baseline algorithm. We can see that CODE on average is superior to all other algorithms. The mean of normalized absolute loss of SGD, AdaGrad, Adam, and Recursive is greater than 1, indicating that these baseline algorithms perform worse than the best constant predictor on average. The reason is clear: on these datasets, no single learning rate can work on all of them. Furthermore, CODE wins Coin by $\sim 0.1984$, which proves that the ODE updates boost the performance of the parameter-free algorithm. CODE also wins Recursive significantly. Overall, CODE essentially guarantees the best performance \emph{without any parameter tuning}. 

\begin{figure}[t]
\centering
\includegraphics[width=0.44\textwidth]{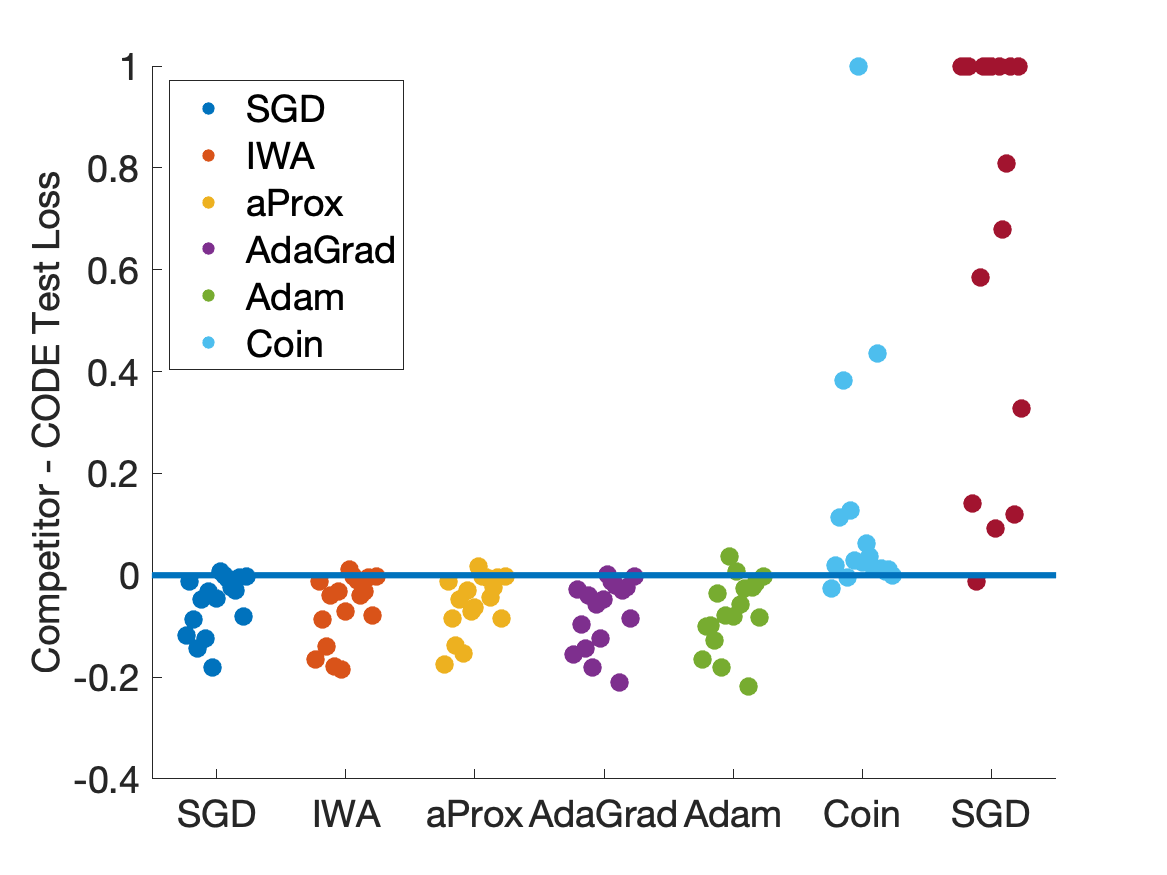} \hfill
\includegraphics[width=0.44\textwidth]{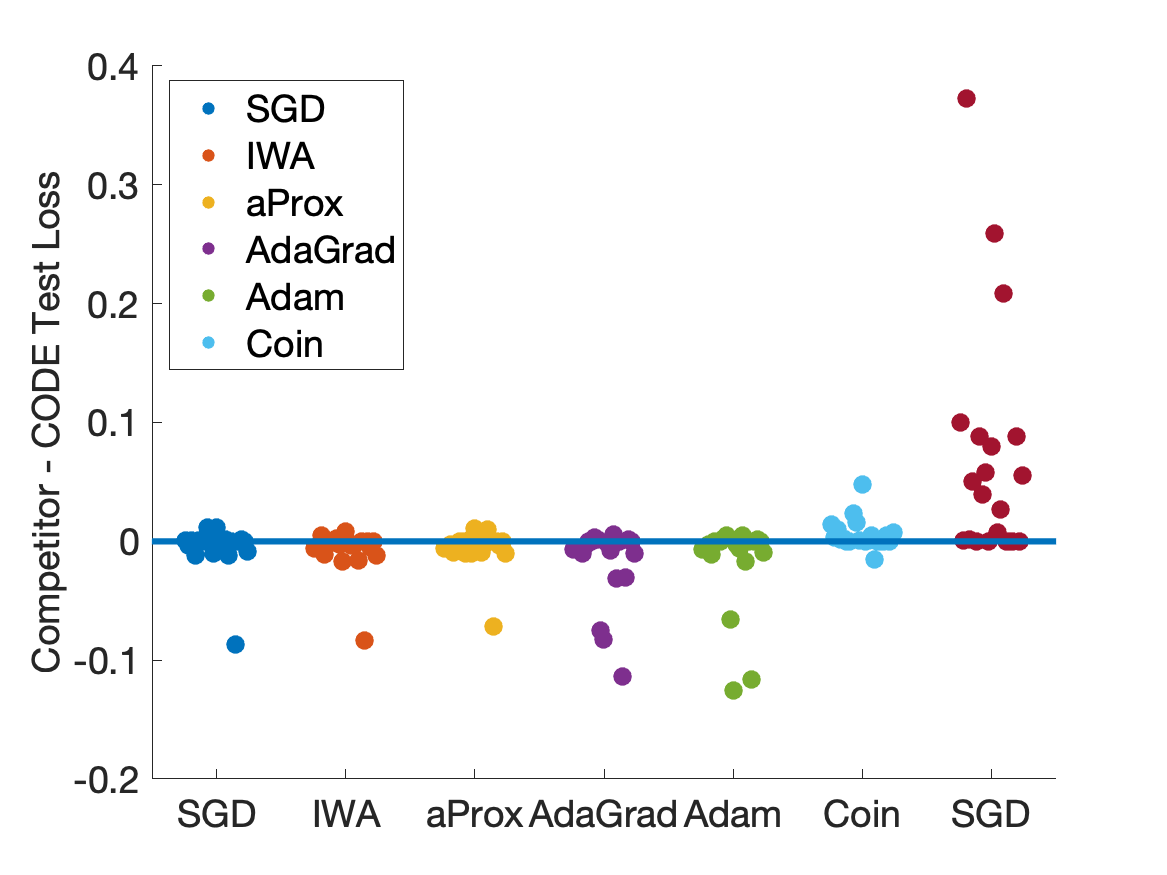}
\caption{Difference between normalized competitor and CODE test losses using a tuned learning rate on 17 regression (1) and 21 classification (2) datasets.}
\label{fig:competitors1}
\end{figure}

We also test the performance of all algorithms on classification problems. In Figure~\ref{fig:competitors2} and in Table~\ref{table:table1}, we can see that all algorithms except Recursive attained similar performance. CODE is worse than AdaGrad on average $\sim 0.0094$ and is better than Coin and Recursive. We performed two-sample paired t-test between CODE and competitors and failed to reject the null hypothesis that the performance of CODE on average is as good as the AdaGrad at the significant level $\alpha=0.05$. 

\paragraph{Tuned Learning Rates.}
We now turn to the case in which we ignore the computational complexity and we tune all the learning rates for SGD, IWA, aProx, AdaGrad, and Adam. For each repetition and dataset, we use the validation set to select the best learning rate, train using that learning rate, test on the test set and report the average of normalized loss. Results are summarized in Figure~\ref{fig:competitors1} and Table~\ref{table:table2}.

As just said, this is a very expensive procedure and not a fair comparison for parameter-free algorithms in terms of computational cost. Yet, in both regression and classification tasks, the performance of CODE and other algorithms except Recursive are close to each other. Remember the fact that CODE achieves this performance without any tuning. As a parameter-free algorithm, CODE only loses over the best algorithm on average $\sim 0.0740$ on regression problems and $\sim 0.0178$ on classification problems. The difference on average is not significant in statistics at the significance level $0.05$. We believe that there are many settings where such loss of accuracy would be negligible compared to what we gain from removing the need to tune the learning rate. It is also instructive to notice how much the performance of the baselines improves when we move from a ``default'' learning rate to a tuned one. In other words, to achieve the best optimal performance with, for example, AdaGrad the parameter tuning cannot be avoided, at least on regression problems.
   
\begin{figure}[t]
\centering
\includegraphics[width=.44\textwidth]{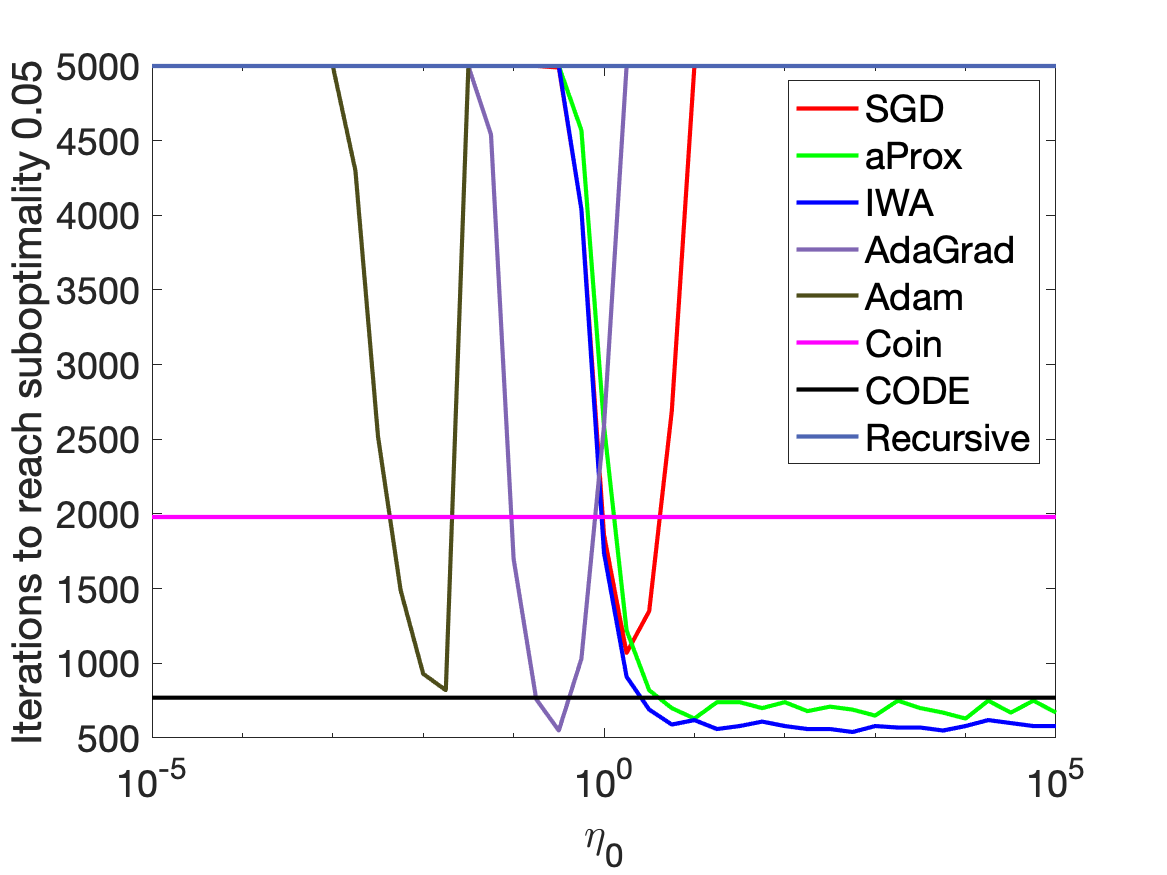} \hfill
\includegraphics[width=.44\textwidth]{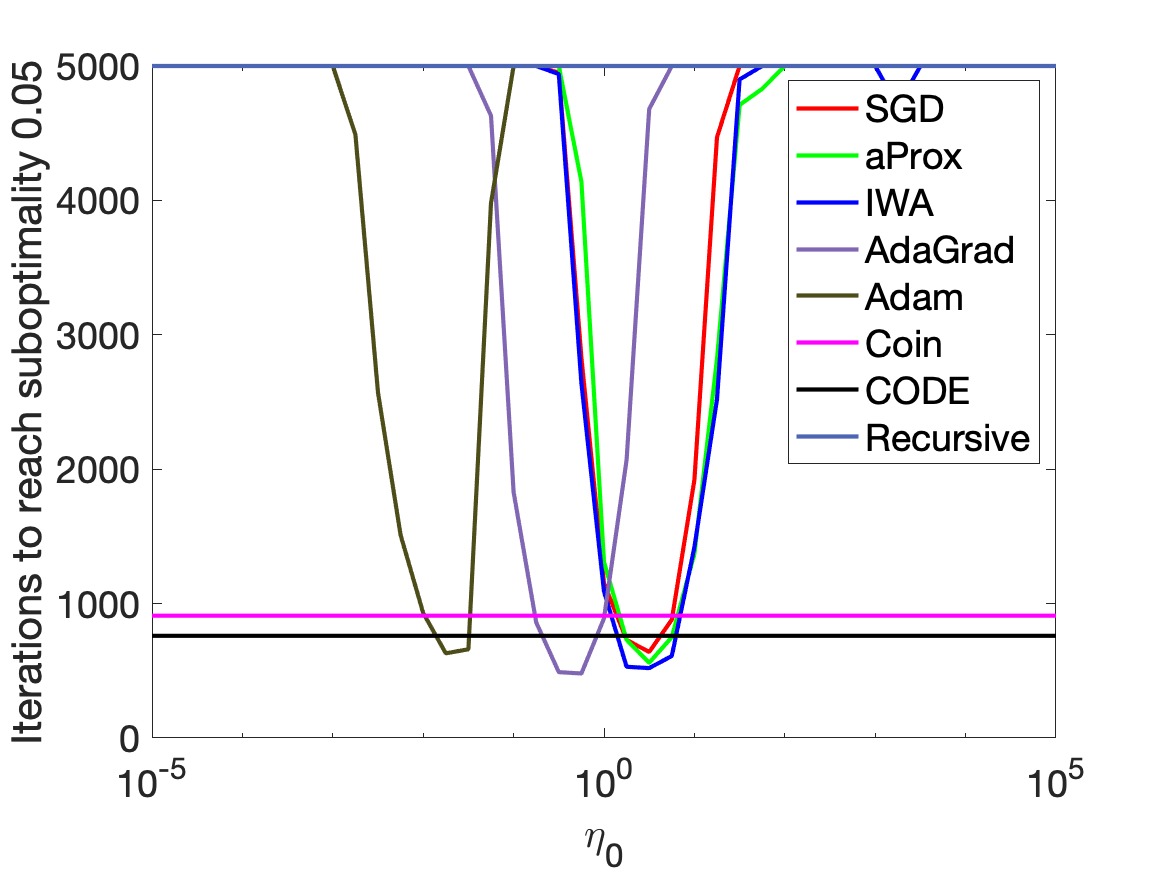}
\caption{Synthetic dataset with absolute loss. Number of iteration to reach $0.05$ suboptimality gap versus initial step sizes $\eta_0$. (1) noiseless setting, (2) $\sigma=1/2$.}
\label{fig:abs}
\end{figure}

\subsection{Sub-Optimality Gap on Synthetic Datasets}
We also generate synthetic data and test the algorithms following the protocol in \citet{AsiD19}, to observe the sensitivity of the algorithms to the setting of the step sizes. Of course, the parameter-free ones do not have any step size to set.
For $A\in \R^{m\times d}$ and $\by \in \R^{m \times 1}$ the objective function we want to minimize is $F(\bx)=\frac{1}{m}\|A \bx - \by\|_1$, which corresponds to a regression problem with the absolute loss. In each experiment, we generate $\bx^\star \sim N(0,I_d)\in \R^d$, and set $\by = A\bx^\star+\sigma \bv$ for $\bv \sim N(0,I_d)$, where $\sigma$ controls the amount of noise. We generate $A$ with uniformly random orthogonal columns, with $m=1000$ and $d=40$. Then, we normalize the $L_2$ norm of each sample. We repeat the above data generation process 10 times and show the average results in the plots. We also consider the classification setting, see similar experiments and plots in the Appendix.

As in \citet{AsiD19}, we study how many iterations are needed to reach a suboptimality gap of 0.05, that is we plot the smallest $k$ such that $F(\bx_k)-F(\bx^\star)\leq 0.05$. In Figure~\ref{fig:abs}, we show the results. As expected, the performance of SGD, Adam, and AdaGrad is extremely sensitive to the setting of the step size. We really need to find the right one, otherwise the convergence slows down catastrophically. Instead, IWA and aProx have a much better range of acceptable step sizes for the noise-free case. However, their advantage almost disappears in the noisy case. On the other hand, we can see that the parameter-free algorithms, CODE, perform very well, with \emph{CODE achieving essentially optimal performance in both regimes}. Moreover, CODE inherits the better performance of aProx in the noise-free case, gaining a big advantage over Coin, but still with a closed-form update. 

\section{Discussion}
\label{sec:conc}

We have presented a new parameter-free method called CODE, which is the first work combining the truncated linear approximation and continuous updates for Coin-Betting optimization algorithms. The empirical results show that CODE can outperform algorithms with a ``default'' learning rate and be very competitive with finely-tuned ones. In future work, we plan to investigate theoretical guarantees for CODE, possibly using our recent analysis of a parameter-free algorithm with an approximation of implicit updates~\citep{ChenCO22}.

\section*{Acknowledgements}
This material is based upon work supported by the National Science Foundation under grants no. 1925930 ``Collaborative Research: TRIPODS Institute for Optimization and Learning'', no. 1908111 ``AF: Small: Collaborative Research: New Representations for Learning Algorithms and Secure Computation'', and no. 2046096 ``CAREER: Parameter-free Optimization Algorithms for Machine Learning''.
\bibliography{learning}

\end{document}